\documentclass[leqno,10pt]{article}
\renewenvironment{description}[1][0pt]
  {\list{}{\labelwidth=0pt \leftmargin=#1
   }}
  {\endlist}

\usepackage[toc,page]{appendix}
\usepackage{enumerate,xspace}
\usepackage{amsmath,xspace,amssymb}
\usepackage{amsthm}
\usepackage{stmaryrd}
\usepackage{verbatim}
\usepackage{fullpage}
\usepackage{amsfonts}
\usepackage{proof}
\usepackage{hyperref}
\usepackage{calc}
\usepackage{tikz}
\usetikzlibrary{positioning}
\usetikzlibrary{decorations.markings}
\tikzstyle{vertex}=[circle, draw, inner sep=0pt, minimum size=6pt]

\input xy
\xyoption{all}
\xyoption{2cell}
\UseAllTwocells
\CompileMatrices

\title{Aspects of Artificial Intelligence: Transforming Machine Learning Systems Naturally}
\author{Xiuzhan Guo\thanks{xiuzhan@gmail.com}}
\date{}

\newcommand{\two}{\ensuremath{{\hbox{\textrm 2}\kern-.25em
        \hbox{\vrule height1.5ex width 0.4pt depth -.2ex}}\kern.2em}\xspace}
\newcommand{\three}{\ensuremath{{\hbox{\textrm 3}\kern-.25em
        \hbox{\vrule height1.5ex width 0.4pt depth -.0ex}}\kern.2em}\xspace}

\newtheorem{theorem}{Theorem}[section]    

\newtheorem{corollary}[theorem]{Corollary}   
   
\newtheorem{preremark}[theorem]{Remark}   
\newtheorem{prexample}[theorem]{Example}   
\newtheorem{proposition}[theorem]{Proposition}

\newtheorem{definition}[theorem]{Definition}

\newenvironment{remark}{\begin{preremark}\rm}{\end{preremark}}
\newenvironment{example}{\begin{prexample}\rm}{\end{prexample}}
\providecommand{\keywords}[1]
{	
 \noindent \textbf{\textit{Keywords --}} #1
}
\begin{document}

\maketitle

\begin{abstract}
In this paper, we study the machine learning elements which we are interested in together
as a machine learning system, consisting of a collection of machine learning elements and a collection of relations between
the elements. The relations we concern are algebraic operations, binary relations, 
and binary relations with composition that can be reasoned categorically.
A machine learning system transformation between two systems is a
map between the systems, which preserves the relations we concern. 
The system transformations given by quotient or clustering, representable functor, and Yoneda embedding 
are highlighted and discussed by machine learning examples. 
An adjunction between machine learning systems, a special machine learning system transformation loop, provides the optimal way of solving problems.
Machine learning system transformations are linked and compared by their maps at 2-cell, natural transformations. 
New insights and structures can be obtained from universal properties and algebraic structures given by monads, which are generated from
adjunctions. 
\end{abstract}


\vspace{1cm}

{\footnotesize \keywords{
Machine Learning, Machine Learning System, Machine Learning System Transformation, Binary Relation, Directed Graph, Category, Functor, Transformation, Quotient, Adjunction, Monad, Descent, Yoneda Embedding
}}
\pagenumbering{arabic}

\vspace{1cm}

Let's begin with the following sentences:
\begin{itemize}
\item
Observed bird nests in trees enduring heavy winds during my morning walking;
\item
My cats dash straight to me at the moment they see me taking their lickable treats from the box, without worrying about the optimal path;
\item
We steer our cars instinctively, without calculating or measuring the exact degrees to turn the steering wheel;
\item
Some people enjoy solving their problems by defining objective functions, constraints, and searching 
for the optimal solutions;
\item
$\cdots$.
\end{itemize}
People might have different feelings and thoughts after reading the words in these sentences and combining the meanings of words together.
The collection $W_0$ of the words in the sentences is discrete.  We connect the words by their interconnections in the sentences to obtain a directed graph $W_1$ and then understand the meanings and insights of these sentences by transforming $W_1$ from one state to another.

Assume that machine learning (ML) aims to learn from data, grow, and perform a class of tasks without explicit instructions.
ML involves a complex and connected collection of ML elements, such as, 
data, algorithms, models, evaluation metrics, monitoring and maintenance, etc.
The objective of this paper is to study ML elements together as an {\em ML system} and map ML systems by {\em ML system transformations}.

\section{Introduction}\label{sect:introduction}
In the age of artificial intelligence, data is from various platforms with multiple formats, noisy, and keeps changing continuously,
which results in tremendous potential for dynamic relationships.
Data is not static but dynamic. Machine learning (ML) elements and systems, driven by data, 
producing new data, must be robust enough to capture the changes.

Natural numbers are not isolated but connected by their mathematical operations, e.g., $+$, $-$, $\times$, and $\div$, so that,
natural numbers can be used not only to count but also to solve real life problems. The set of all natural numbers, along
with the operations, forms an algebraic system and so one can study numbers and their relations together by their properties and extend the system 
to more complex system and
solve more complex problems naturally.
Hence the relations (operations) between natural numbers and their properties make more sense than the isolated numbers.

In ML, algorithms learn from the data one inputs and can only learn effectively if the data is in the format required, clean and complete.
ML models are driven by data and on the other hand, ML models generate data usually.
Real world datasets are usually with multiple formats, from multiple silos. These datasets are prepared and transformed to train their ML models.
Therefore, not only are data and ML models connected but also are there relations among datasets and between ML models.
Hence ML elements are not isolated but connected together with certain structures,
e.g., operations, relations, and compositional relations. 
These relations make more sense than the isolated ML elements,
similar to the natural numbers.

Data changes constantly. ML models, driven by data, are tested and retrained to ensure them remain accurate, relevant, and effective during data changing.
So all elements and their relations in an ML system must be modified together coherently.
All ML elements
and the relations between the elements, which we concern, must be viewed together to form an {\em ML system}.

Real world problems can be solved by modelling them mathematically and implementing the models into computational tools.
Some problems are challenges in one mathematical area but can be solved through mapping them to another mathematical environment, e.g.,
mapping topological problems and number theory problems to algebraic settings.
Similarly, some problems might be easier to be solved in one ML system than another. To find a reliable ML system for certain class of problems,
we need to compare some ML systems and map one system to another system without breaking the existing relations.
Therefore, we also consider how to map one ML system to another.

In Section \ref{sect:mlsystemtransformation}, we first consider the collection ${\bf M}$ of all ML elements and a collection $R$ of their relations we concern 
together as an {\em ML system} $({\bf M},R)$.
Then we link and compare ML systems using relations preserving maps, {\em ML system transformations}.
In this section, the ML element relations we are interested in are binary algebraic operation, directed graph, directed graph with identities and composition.
Clustering aims to group a class of objects in such a way that objects in the same cluster are more similar to each other.
Clustering amounts to partitioning or an equivalent relation on the class of objects or a surjective map from the class 
to the quotient space of the equivalent relation. If clustering is {\em compatible} with the relations of an ML system,
then we have a quotient ML system by identifying elements in the same cluster and 
a quotient ML system transformation from the original ML system to its quotient ML system functorially.

Sets are concrete mathematical objects.
Given an ML system $({\bf M},R)$ and an ML element $M\in {\bf M}$,
we have a map $\textrm{hom}(-,M)$ from $({\bf M},R)$ to ${\bf Set}$, sending $X\in {\bf M}$ to the set of all edges from $X$ to $M$.
The corresponding, sending $M$ to $\textrm{hom}(-,M)$, is called {\em Yoneda embedding}.
The $\textrm{hom}(-,M)$ is set valued and determined totally by $M$, called {\em reprsentable} functor/transformation when $({\bf M},R)$
has composition and identities. 
These representable functors/transformations, e.g., $\textrm{hom}(-,M)$, provide the optimal ways to understand elements, e.g., $M$ of ML system $({\bf M},R)$,
by their representable functors in the category of all set valued functors (presheaves),
through Yoneda embedding.
We shall show that Yoneda embedding, along with Yoneda lemma, plays a crucial role in ML system transformation and representation.

An ML system can be transformed by different ways. Let's collect all ML system transformations between two ML systems together. Now it is natural to ask what the relations are between these transformations
and how to compare these transformations. In Section \ref{sect:transformcompare}, the questions are answered categorically:
The structure-preserving maps between ML system transformations are categorical natural transformations and these natural transformations are flatted to a preorder on 
ML system transformations.

ML intends to understand and summarize the existing knowledge from data to grow, predict, and create (new) insights from data.
We employ category theory to format and reason ML systems and ML system transformations naturally.
An ML system transformation maps problems from one ML system to another where the problems mapped are easier to solve and then
the solutions are mapped back to the original system. Hence an ML transformation loop is needed.
Categorical concept ``adjunction" describes the most efficient solution to problems involving transforming problems and solutions naturally.
A monad $T$ on an ML system $({\bf M},R)$, is an endo system transformation $T:({\bf M},R)\rightarrow ({\bf M},R)$ with monoid like structure.
$T$ acts on $({\bf M},R)$ and outputs algebraic structures, $T$-algebras, to $({\bf M},R)$
and so $({\bf M},R)$ obtains algebraic structures through $T$.
An adjunction gives rise to a monad and every monad arises in this way. 
Also, andjunctions can be defined by the universal property that confirms 
the existence and uniqueness of the gap map/link, which can be used to link ML elements.
In Section \ref{sect:mladjunction}, 
the adjunctions between ML systems we highlight include Yoneda embedding, monad algebras, free structures, change of base functors. 

Finally, we complete the paper with our concluding remarks
in Section \ref{sect:conclusions}.
 
\section{Machine Learning Systems and Transformations}\label{sect:mlsystemtransformation}
ML elements are the foundational components and blocks that can be used to build ML systems. 
Essential  ML elements include data, features, algorithms, models, performance metrics, validation, testing, deployment, outputs, etc.
Let ${\bf M}$ be a collection of ML elements one concerns. 
The elements in ${\bf M}$ are not isolated but connected by the collection of relations between the elements.
For instance, the collection of relations can be specified by algebraic operations, certain dependencies and relations on ${\bf M}$ so that the elements
work together to enable the ML systems to learn, grow, and perform tasks. 
A collection ${\bf M}$ of ML elements and a collection of relations we concern form an {\em ML system}.

Data is flowing. The elements and the relations in an ML system, driven by data, must be updated and transformed to 
fit the present setting dynamically.
The map, preserving the relations concerned, between ML systems is {\em an ML system transformation}.

\begin{definition}\label{def:transformation}\textrm{(Machine learning system and transformation)}
\begin{enumerate}
\item
An {\em ML system} $({\bf M},R)$ consists of a collection ${\bf M}$ of ML elements and 
a collection $R$ of relations between the elements.
Write $e_1Re_2$ or $(e_1,e_2)\in R$ or $e_1\rightarrow e_2$ if $e_1$ and $e_2$ are related by $R$, for $e_1,e_2\in {\bf M}$.
\item 
Let $({\bf M}_1,R_1)$ and $({\bf M}_2,R_2)$ be ML systems.  An {\em ML system transformation} $T$ from $({\bf M}_1,R_1)$ to $({\bf M}_2,R_2)$ is a function 
$T:{\bf M}_1\rightarrow {\bf M}_2$ that preserves the collection $R_1$ of relations: $e_1R_1e_2$ implies $T(e_1)R_2T(e_2)$, denoted by $T:({\bf M}_1,R_1)\rightarrow ({\bf M}_2,R_2)$.
\end{enumerate}
\end{definition}

\begin{remark}\label{remark:transformation}
Let $T:({\bf M}_1,R_1)\rightarrow ({\bf M}_2,R_2)$ be an ML system transformation. 
\begin{enumerate}
\item
If $R_1$ is given a (partial) binary operation, e.g., table join on a collection of data tables, then $R_1$ can be viewed as a ternary relation.
Assume that $R_1$ is given by a (partial) binary operation $\circ$ and $R_2$ by $\star$ respectively,
$T:({\bf M}_1,R_1)\rightarrow ({\bf M}_2,R_2)$ can be considered as a homomorphism, a structure preserving function $T:{\bf M}_1\rightarrow {\bf M}_2$,
namely, $T(e_1\circ e_2)= T(e_1)\star T(e_2)$.
\item
If ML systems $({\bf M}_1,R_1)$ and $({\bf M}_2,R_2)$ have only some general relations, e.g., dependencies, similarities, implications, etc., between their elements, then these ML systems can be modelled 
by (multi)directed graphs and so the ML transformations between the ML systems are given by directed graph homomophisms, namely, a function $T: {\bf M}_1\rightarrow {\bf M}_2$
that takes each edge (relation) $e_1\rightarrow e_2$ in ${\bf M}_1$ to an edge (relation) $T(e_1)\rightarrow T(e_2)$ in ${\bf M}_2$.
\item
If ML systems $({\bf M}_1,R_1)$ and $({\bf M}_2,R_2)$ have the transitive and associative relations and identity relations, then both $({\bf M}_1,R_1)$ and $({\bf M}_2,R_2)$ can be modelled
by directed graphs with identities and composition, which are {\em categories}, a general mathematical structure. 
An ML system transformation $T:({\bf M}_1,R_1)\rightarrow ({\bf M}_2,R_2)$ is a {\em functor}. Hence ML systems can be reasoned categorically. 
See Appendix for the basic notations, concepts, and results of relation, directed graph, and category theory.
\end{enumerate}
\end{remark}

Algebraic or graph transformation between ML systems that have algebraic operations or binary relations, can be factored as a surjective to a quotient space, followed by an injective transformation 
by the similar process in \cite{ghlr} at the set level.
An ML system that has a compositional relation and forms a category can be quotiented by either a congruence equivalence relation on its hom sets or an equivalence relation on objects.
See Subsection \ref{subsect:quotientcat} for the quotient category details.

Let $({\bf M}, R)$ be an ML system and $\rho$ an equivalence relation on ${\bf M}$.
$\rho$ is {\em compatible} with $R$ if $R$ can be induced to the equivalence relation $R_{\rho}$ on ${\bf M}/\rho$, 
namely, $e_1Re_2\Rightarrow [e_1]_{\rho}R_{\rho}[e_2]_{\rho}$ is well-defined.
\begin{proposition}\label{prop:quotient}
Let $({\bf M}_1, R_1)$ be an ML system and $\rho$ an equivalence relation on ${\bf M}_1$.
Suppose that $\rho$ is compatible with $R_1$. 
\begin{enumerate}
\item
ML system $({\bf M}_1, R_1)$ is transformed to its quotient ML system $({\bf M}_1/\rho, {R_1}_{\rho})$ by 
the obvious canonical ML system transformation
$$Q_{\rho}:{\bf M}_1\rightarrow {\bf M}_1/\rho$$
sending $f:e_1\rightarrow e_2$ to $[f]_{\rho}:[e_1]_{\rho}\rightarrow [e_2]_{\rho}$.
\item
If $\sigma$ is an equivalence relation on ${\bf M}_1$ and compatible with $R_1$ such that $\rho\subseteq \sigma$, then there is a unique surjective
ML system transformation $(\rho\leq\sigma)^*: {\bf M}_1/\rho\rightarrow {\bf M_1}/\sigma$ such that
$$\xymatrix@C=1.6pc{
& {\scriptstyle {\bf M}_1}\ar[dl]_{Q_{\rho}} \ar[dr]^{Q_{\sigma}}\\
{\scriptstyle {\bf M}_1/\rho} \ar@{..>}[rr]^{(\rho\leq\sigma)^*} && {\scriptstyle {\bf M}_1/\sigma}
}$$
commutes.
\item
Each ML system transformation $T:({\bf M}_1,R_1)\rightarrow ({\bf M}_2,R_2)$, which preserves the congruence equivalence relation $\rho$,
that is, $(f,g)\in \rho$ implies $T(f)=T(g)$,
factors through $Q_{\rho}$, followed by a unique induced ML system transformation $T_{\rho}:{\bf M}_1/\rho\rightarrow {\bf M}_2$
$$\xymatrix@C=1.6pc{
{\scriptstyle {\bf M}_1}\ar[rr]^T\ar[dr]_{Q_{\rho}} && {\scriptstyle {\bf M}_2}\\
& {\scriptstyle {\bf M}_1/\rho} \ar[ur]_{T_{\rho}}
}$$
\item
If $T(\rho)\subseteq \sigma$, then there is a unique ML system transformation $T^*:{\bf M}_1/\rho\rightarrow {\bf M}_2/\sigma$
such that
$$\xymatrix@C=1.6pc{
{\scriptstyle {\bf M}_1}\ar[r]^T\ar[d]_{Q_{\rho}} & {\scriptstyle {\bf M}_2}\ar[d]^{Q_{\sigma}}\\
{\scriptstyle {\bf M}_1/\rho} \ar@{..>}[r]^{T^*}  &  {\scriptstyle {\bf M}_2/\sigma}\\
}$$
commutes. If $T$ is surjective and so is $T^*$.
\end{enumerate}
\end{proposition}

\begin{example}\label{item:mlsystrans_quotient}{\em Cluster and represent ML elements from multiple points of view:}
Duplicates or similar ML elements in an ML system are often clustered and quotiented to a new ML system.
After transforming an ML system by the clusters, we want to represent each equivalence class (cluster) by an ML element on each cluster.

For example, given a sequence 
$$S=\{s_1,\cdots,s_n\textrm{ of temporal transaction records}\},$$ 
to increase the predictability of the next record or a few next records, one may
group the records and compute the average of each cluster, for instance, monthly cluster. Hence the equivalence relation $\rho$ on $S$ is given by:
$$(s_i,s_j)\in\rho\Leftrightarrow s_i\textrm{ and }s_j \textrm{ were transacted at the same month}.$$ 
The representation function on clusters, e.g., the average $avg_{\rho}$ can be chosen to represent each cluster. Therefore we have:
$$S\stackrel{Q_{\rho}}{\longrightarrow} S/\rho \stackrel{avg_{\rho}}{\longrightarrow}({\bf R},\preceq),$$
where $\preceq$ is a partial order on ${\bf R}$ and compatible with the represent function $avg$ to induce $avg_{\rho}$.

The uncertainty/randomness of an ML element can be viewed from different points of view by clustering and representing.
The outcomes of a variable, e.g., coin flip, can be either any point from $\{$head, tail$\}$ at the point level or a certain set $\{$head, tail$\}$ at the distribution level.

Let $\Omega$ be a sample space of a random variable $X$ and $bin$ an equivalence relation on $\Omega$. $bin$ clusters $\Omega$
to a family of disjoint subsets. If $dist:\Omega \rightarrow V$ is a function that is compatible with $bin$ and represents the clusters of $bin$, then there is a transformation:
$$\Omega \stackrel{Q_{bin}}{\longrightarrow}\Omega/bin\stackrel{dist}{\longrightarrow} V.$$

Since equivalence relations (clustering or surjective maps) have the partial order $\leq$ given by $\subseteq$, we have the following commutative diagrams:
$$\xymatrix@C=1.6pc{
{\scriptstyle S} \ar[rr]^{Q_{\rho}}\ar@{=}[d] && {\scriptstyle S/\rho} \ar[d]^{(\rho\leq\sigma)^*}\\
{\scriptstyle S} \ar[rr]^{Q_{\sigma}} && {\scriptstyle S/\sigma} \\
}$$
and
$$\xymatrix@C=1.6pc{
{\scriptstyle \Omega} \ar[rr]^{Q_{\rho}}\ar@{=}[d] && {\scriptstyle \Omega/\rho} \ar[d]^{(\rho\leq\sigma)^*}\\
{\scriptstyle \Omega} \ar[rr]^{Q_{\sigma}} && {\scriptstyle \Omega/\sigma}\\
}$$
\end{example}

\begin{example}\label{item:mlsystrans_nlp}{\em Word2vec, Word2fun:}
Word embedding intends to map all words in a large corpus to a vector space,
with the relations between words, e.g., semantic similarity, syntactic similarity,
contextual similarity, analogical relationships, etc., being preserved.
Word2vec \cite{mccd,msccd} is a popular machine learning technique for learning word embeddings from a large text corpus.
Let ${\bf W}_0$ be the set of all words in a corpus. Applying Word2vec to ${\bf W}_0$, one obtains a function
$w2v_0:{\bf W}_0\rightarrow {\bf R}^n$. Since there exist the relations, e.g., similar meaning, between the words in ${\bf W}_0$, 
${\bf W}_0$ is enriched to a directed graph ${\bf W}_1$ and $w2v_0$ is lifted to $w2v:{\bf W}_1\rightarrow {\bf R}^n$,
where $n$ is a natural number and ${\bf R}^n$ is ordered partially by the {\em closeness} (neighbourhood). 
If the words that have similar meaning are closer in the real vector space ${\bf R}^n$, then one has an ML system transformation:
$$w2v:{\bf W}_1\rightarrow {\bf R}^n$$
and so NLP problems in ${\bf W}_1$ might be solved in a vector space ${\bf R}^n$ using the structures and properties of ${\bf R}^n$ 
through the transformation $w2v$. For example, 
Jiang et al. \cite{jav} showed that the semantic independence structure of language
are naturally represented by {\em partial orthogonality} in vector space ${\bf R}^n$.
Since ${\bf R}^n$ is a {\em linear} vector space and ${\bf W}_1$ might have more complex word relations that are difficult to represent in 
${\bf R}^n$,  Mani \cite{m} introduced multi-vectors and geometric algebra to embed words in ${\bf W}_1$.

Sets are among the most fundamental objects in mathematics and many structures, e.g., graphs, algebras, topologies, geometries, can associate with them.
Yuan \cite{yuan} considered set valued  (representable) functors as tasks.
Assume that ${\bf W}_1$ is a directed graph. We first transform ${\bf W}_1$ using $path$ functor to form a category $path({\bf W}_1)$ and 
transform $path({\bf W}_1)$ to representable functors in the category of presheaves using Yoneda embedding:
$$\xymatrix@C=1.6pc{
{\scriptstyle {\bf W}_1} \ar[r]^<<<<{\scriptstyle path} & {\scriptstyle path({\bf W}_1)} \ar[r]^<<<<{\scriptstyle Q_{\sim}} & {\scriptstyle path({\bf W}_1)/\sim} \ar[r]^<<<<Y & 
{\scriptstyle {\bf Set}^{(path({\bf W}_1)/\sim)^{\textrm{op}}}}\\
}$$
where $wordchain_1\sim wordchain_2$ if both have the similar meaning. See \ref{subsect:yoneda} for the descriptions of presheaves, representable functors, and Yoneda embedding.

For example, the discrete subset $W_0=\{\textrm{apples,\;eat,\;i,\;like,\;love}\}\subseteq {\bf W}_0$  is a set of isolated words but can be directed graph enriched by the sentences of a corpus:
$$
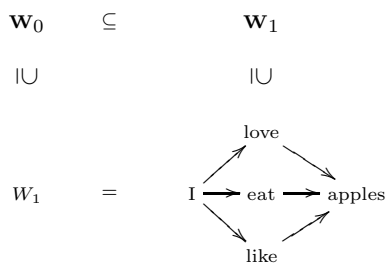
\xymatrix@R=1pc@C=1pc{
{\scriptstyle \textbf{W}}_0\ar@{}[dd]|{\rotatebox{-90}{$\supseteq$}} \ar@{}[rrr]|{\subseteq}&&&& {\scriptstyle {\bf W}}_1 \ar@{}[dd]|{\rotatebox{-90}{$\supseteq$}} \\
&&&&\\
 && && {\scriptstyle  \textrm{love}}  \ar[dr]\\
{\scriptstyle W_1} \ar@{}[rrr]|{=} &&& {\scriptstyle  \textrm{I}} \ar[r] \ar[ur]  \ar[dr]&  {\scriptstyle  \textrm{eat}}  \ar[r]&  {\scriptstyle  \textrm{apples}}\\
&& && {\scriptstyle  \textrm{like}}  \ar[ur]\\
}$$

Applying $path$ to $W_1$,
we have the path edges from I to apples:
$$
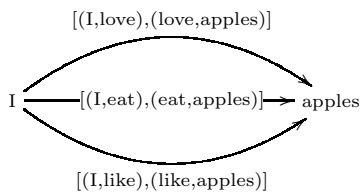
\xymatrix@R=1pc@C=1pc{
 {\scriptstyle  \textrm{I}} \ar@/^2pc/[rrrrrr]^{\scriptstyle  \textrm{[(I,love),(love,apples)]}} \ar@/_2pc/[rrrrrr]_{\scriptstyle  \textrm{[(I,like),(like,apples)]}}
\ar[rrrrrr]|{\scriptstyle  \textrm{[(I,eat),(eat,apples)]}} &&&&&& {\scriptstyle  \textrm{apples}}
}$$
Quotienting out duplicates, we have:
$$
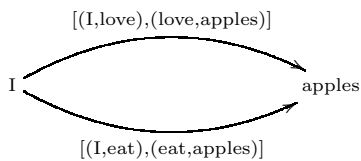
\xymatrix@R=1pc@C=1pc{
 {\scriptstyle  \textrm{I}} \ar@/^1.5pc/[rrrrrr]^{\scriptstyle  \textrm{[(I,love),(love,apples)]}} \ar@/_1.5pc/[rrrrrr]_{\scriptstyle  \textrm{[(I,eat),(eat,apples)]}}
 &&&&&& {\scriptstyle  \textrm{apples}}
}$$

Clustering/Quotient provides multiple points of view on chains of words at different layers.
For instance, the following quotient system transformations given by $\rho_1$ and $\rho_2$:
$$
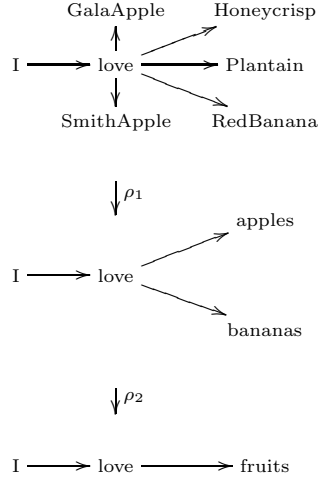
\xymatrix@R=0.8pc@C=0.8pc{
& {\scriptstyle  \textrm{GalaApple}} & {\scriptstyle \textrm{Honeycrisp}}\\
{\scriptstyle \textrm{I}} \ar[r] & {\scriptstyle \textrm{love}} \ar[d] \ar[dr] \ar[u]  \ar[ur] \ar[r] & {\scriptstyle \textrm{Plantain}}\\
& {\scriptstyle \textrm{SmithApple}}& {\scriptstyle \textrm{RedBanana}} \\
& \ar[d]^{\rho_1}\\
& & {\scriptstyle \textrm{apples}}\\
{\scriptstyle \textrm{I}} \ar[r] & {\scriptstyle \textrm{love}}  \ar[dr] \ar[ur] & \\
& & {\scriptstyle \textrm{bananas}} \\
& \ar[d]^{\rho_2}\\
& {}\\
{\scriptstyle \textrm{I}} \ar[r] & {\scriptstyle \textrm{love}}  \ar[r] & {\scriptstyle \textrm{fruits}} \\
}$$
show how the chains of words are aggregated by clustering/ equivalence relations $\rho_1$ and $\rho_2$, functorially.

${\bf Set}^{(path({\bf W}_1)/\sim)^{\textrm{op}}}$ (the category of presheaves) has more complicated structure than ${\bf R}^n$ and behaves like sets.
Given two representable functors $F_1$ and $F_2$, represented by words $w_1,w_2$ in ${\bf W}_1$, respectively, by Yoneda Lemma  $F_1$ and $F_2$ are the same up to isomorphism
if and only if $w_1$ and $w_2$ are the same.  The process of representing words into their presentable functors, is functorial: for each word corresponding $f:{\bf W}_1\rightarrow {\bf W_2}$, 
$$\xymatrix@C=1.6pc{
{\scriptstyle {\bf W}_1} \ar[r]^<<<<{path} \ar[d]_f & {\scriptstyle path({\bf W}_1)} \ar[r]^<<<<{Q_{\sim}} \ar[d]\ar[d]_{path(f)}& {\scriptstyle path({\bf W}_1)/\sim} \ar[r]^<<<<Y \ar[d]_{path(f)^*}
   & {\scriptstyle {\bf Set}^{(path({\bf W}_1)/\sim)^{\textrm{op}}}}\ar[d]^{Y(path(f)^*)}\\
{\scriptstyle {\bf W}_2} \ar[r]^<<<<{path} & {\scriptstyle path({\bf W}_2)} \ar[r]^<<<<{Q_{\sim}} & {\scriptstyle path({\bf W}_2)/\sim} \ar[r]^<<<<Y 
  & {\scriptstyle {\bf Set}^{(path({\bf W}_2)/\sim)^{\textrm{op}}}}\\
}$$
is commutative, forgetting the categorical composition.
Therefore NLP problems in ${\bf W}_1$ can be transformed to the category of presheaves naturally, where word relations are enriched by set set valued functors functorially.
\end{example}

\begin{example}\label{data22model}{\em Data and model loop:}
Assume that ML models are driven by data. Given a data $d$, one applies ML system transformations, e.g., 
clustering or quotienting, summarizing, aggregating, etc., on the data $d$ to learn an ML model $m$, which builds the relations between inputs and outputs, 
represents data $d$ and behaves like a mathematical function.
On the other hand, given an ML model $m$, it can be thought as a colimit of repsentable functors and
produces new data by inputting data, modelling the relations between inputs and outputs.
Let {\bf Data} be a collection of data sets and ${\bf Model}$ a collection of ML models. Then one has maps 
$$M:{\bf Data}\rightarrow {\bf Model}$$
and 
$$P:{\bf Model}\rightarrow {\bf Data}.$$
${\bf Data}$ has a collection of data relations 
$$R_D=\{\textrm{join, combine, select, merge, join conditions,}$$
$$\textrm{ match, similarity,}\cdots \}$$
and ${\bf Model}$ a collection of model relations 
$$R_M=\{\textrm{compose, combine, match, similarity,}\cdots\},$$
respectively. Choose proper operations or relations $R_d\in R_D$  and $R_m\in R_M$, based on the problem one wants to solve, so that
$$\xymatrix@R=1pc@C=1pc{
{\scriptstyle ({\bf Data},R_d)} \ar@/_1pc/[rr]_{M} && {\scriptstyle ({\bf Model},R_m)}\ar@/_1pc/[ll]_P
}$$
becomes an ML system transformation loop. We shall study when the loop becomes an adjunction and generates new structures in Subsection \ref{subset:optadj}.
\end{example}

ML systems and transformations arise everywhere. More examples of ML systems and transformations are listed in the following Example \ref{examp:mlsystrans}.
\begin{example}\label{examp:mlsystrans}
\begin{enumerate}
\item\label{item:mlsystrans_matchmerge}
{\em ML entity match and merge:}
Assume that ${\bf E}$ is a set of ML entity records, e.g., data sets, data workflows, etc.
An ML entity resolution system $({\bf E},\approx, \langle\;\rangle)$ \cite{BGSWW, gbkbl} consists of a set of ML entity records ${\bf E}$,
a {\em match} function $\approx:{\bf E}\times {\bf E}\rightarrow \{true,false\}$, where $\approx(e_1, e_2)=true$ means that $e_1$ matches $e_2$, modelling match relations.
For instance, $e_1$ is similar to $e_2$ if $e_1$ and $e_2$ have an overlap. $\langle e_1, e_2\rangle$ combines $e_1$ and $e_2$ together by identifying their overlap when 
$e_1$ matches $e_2$.
Then $({\bf E},\approx, \langle\;\rangle)$ gives rise to an algebraic system, a partial groupoid $({\bf E}, \circ)$,
where $e_1\circ e_2= \langle e_1,e_2\rangle$ when $\approx(e_1,e_2)=true$ and undefined, otherwise.
Clearly, an ML entity resolution system $({\bf E},\approx, \langle\;\rangle)$ leads to an ML system $({\bf E},\circ)$ with a partial algebraic operation $\circ$.

\item\label{item:mlsystrans_zoomdp}
{\em Zoom data provenance by clustering:}
Recall that data provenance aims to provide a historical record of data origins and transformations associated with data.
Data provenance knowledge can be represented as a collection ${\bf D}$ of data elements and a collection of relations $L$ between the elements \cite{abcdkv} and so it forms 
an ML system $({\bf D},L)$. Data element relations can be rolled up or down according to the data hierarchy, by using equivalence relations $\tau$ and so data
provenance is zoomed, aggregated, queried, and visualized at multiple levels driven by $\tau$:
$$({\bf D},L) \stackrel{Q_{\tau}}{\longrightarrow}({\bf D}/\tau,L/\tau)\stackrel{query}{\longrightarrow}{\bf D}.$$

\item\label{item:mlsystrans_world}
{\em World2vec, World2fun:}
Not only are words connected but also is everything in the World linked, interactive and dynamic.
Let ${\cal W}$ be the collection of elements (things) in the World and ${\cal R}$ the collection of relations one concerns, between the things.
Similar to Word2vec, some subsets of ${\cal W}$, e.g., graph, ontology, were represented into vector spaces \cite{chjhah,ncvclj}.
Since ${\cal W}$ may have more complex relations than vector spaces, by the similar processes in Example \ref{item:mlsystrans_nlp},
we transform ${\cal W}$ to representable functors in the category of presheaves using Yoneda embedding:
$$\xymatrix@C=1.4pc{
{\scriptstyle {\cal W}} \ar[r]^<<<<{path} & {\scriptstyle path({\cal W})} \ar[r]^<<<<{Q_{\sim}} & {\scriptstyle path({\cal W})/\sim} \ar[r]^<<<<Y 
 & {\scriptstyle {\bf Set}^{(path({\cal W})/\sim)^{\textrm{op}}}}\\
}$$

\item\label{item:mlsystrans_slice}
{\em Slice ML system, coslice ML system:}
Let $({\bf M},R)$ be an ML sysytem that is a category (directed graph with identities and composition) and $N$ an ML element in ${\bf M}$. Define $({\bf M}/N,R/N)$ by
\begin{itemize}
\item
${\bf M}/N=\{x:X\rightarrow N\;|x\in R\;\}$,
\item
a relation from $x:X\rightarrow N$ to $y:Y\rightarrow N$ in $R/N$ is a relation $e:X\rightarrow Y$ in $R$ such that
 $$\xymatrix@C=1.6pc{
{\scriptstyle X}\ar[rr]^e \ar[dr]_x && {\scriptstyle Y}\ar[dl]^y\\
&{\scriptstyle N}\\
}$$
commutes. Then $({\bf M}/N,R/N)$ is  the {\em slice} ML system of $({\bf M},R)$ over $N$.
\end{itemize}
If $e:N_1\rightarrow N_2$ is a relation in $R$, there are ML system transformations 
$$e_!:({\bf M}/N_1,R/N_1)\rightarrow ({\bf M}/N_2,R/N_2)$$ 
and 
$$e^*:({\bf M}/N_2,R/N_2)\rightarrow ({\bf M}/N_1,R/N_1)$$ given by composition and pullback, respectively. 
Dually, one defines coslice ML systems and their adjunctions.
We shall see the details of the change of base functors in \ref{subsect:descent}.

\item\label{item:mlsystrans_optimization}
{\em Optimal search:}
An optimal search $({\bf C}, \textrm{ob})$ over ${\bf R}^n$, consisting of a feasible space ${\bf C}\subseteq {\bf R}^n$ and an objective function $\textrm{Ob}:{\bf R}^n\rightarrow {\bf R}$, 
aims to search the best element(s) in the feasible space ${\bf C}$, with respect to certain criteria. 
One of mathematical optimization problems is as follows.
$$
\begin{aligned}
\textrm{minimize} \quad & f_0(x) & \\
\textrm{subject to} \quad & f_i(x)\leq b_i, & i =1, \ldots, n,\\
\end{aligned}
$$
where $f_j:{\bf R}^n\rightarrow {\bf R}, j=0,1,\ldots,n$ are functions and $b_i\in {\bf R}, i=1,\ldots,n$, 
It amounts to the optimal search $({\bf C},\textrm{Ob})$, 
where ${\bf C}=\{x\in {\bf R}^n\;|\;f_i(x)\leq b_i,i=1,\ldots,n\}$ and $\textrm{Ob}:{\bf R}^n \rightarrow {\bf R}$ is defined by $f_0$.
Optimal search variables from a feasible space, can be discrete, categorical, or continuous. An objective function may have its {\em optimal type}, e.g., minimum, maximum, 
inflection points on the function.
 
Let ${\bf OptS}$ be a collection of optimal searchs over ${\bf R}^n$. Since $2^{{\bf R}^n}$ is a poset with $\subseteq$ and objective functions can be compared point wise.
${\bf OptS}$ is a poset and so an ML system. 
An ML system transformation $T:{\bf OptS}_1\rightarrow {\bf OptS}_2$ is a poset homomorphism (monotone function) .
\end{enumerate}
\end{example}
As explained in Examples \ref{item:mlsystrans_nlp} and \ref{examp:mlsystrans}, we have
\begin{proposition}
Let $({\bf M}, R)$ be an ML system. Then there are ML system transformations
 given by the compositions of the following ML system transformations:
$$\xymatrix@C=1pc{
{\scriptstyle path({\bf M}, R)} \ar[r]^<<<{Q_{\sim}} & {\scriptstyle path({\bf M}, R)/\sim} \ar[r]^<<<{Y} 
& {\scriptstyle {\bf Set}^{(path({\bf M}, R)/\sim)^{\textrm{op}}}}\\
}$$
and
$$\xymatrix@C=1.8pc{
{\scriptstyle ({\bf M}, R)}\ar[r]^<<<<<{path} & {\scriptstyle Upath({\bf M}, R)} \ar[r]^<<<<<{UQ_{\sim}} & {\scriptstyle U(path({\bf M}, R)/\sim))} \ar[r]^<<<<<{UY}  &
 {\scriptstyle U({\bf Set}^{(path({\bf M}, R)/\sim)^{\textrm{op}}}})\\
}$$
where $U:{\bf Cat}\rightarrow {\bf Grph}$ is the forgetful functor, forgetting categorical composition and identity edges.
\end{proposition}

Mathematical objects are determined by-and understood by-the network of relationships they
enjoy with all the other objects of their species \cite{maz}.
Yoneda embedding represents each ML element $E$ to its hom set system transformation $\textrm{hom}(-,E)$ which maps each element $X$ to the set $\textrm{hom}(X,E)$
of all relations between $X$ and $E$. Hence we use ML element relations to study ML elements by Yoneda embedding.

\section{Transforming and Comparing ML System Transformations}\label{sect:transformcompare}
In this section and Section \ref{sect:mladjunction},
we assume that ML systems are categories: directed graphs with identities and composition so that categorical results are applicable.
Hence ML system transformations between ML systems are functors and relations/maps between functors are natural
transformations categorically. See \ref{subsect:functor} for the details of functors and natural transformations.

Let $({\bf M}_1,R_1)$ and $({\bf M}_2,R_2)$ be two ML systems and let
$({\bf M}_2,R_2)^{({\bf M}_1,R_1)}$ be specified by
\begin{itemize}
\item
objects: the collection of ML system transformations from $({\bf M}_1,R_1)$ to $({\bf M}_2,R_2)$,
\item
relations: the collection of natural transformations between ML system transformations.
\end{itemize}
Then $({\bf M}_2,R_2)^{({\bf M}_1,R_1)}$ is an ML system.

Recall that a preorder is a reflexive and transitive binary relation.
ML system transformations are preordered naturally by flatting natural transformations between two ML system transformation.
\begin{proposition}
Let $({\bf M}_1,R_1)$ and $({\bf M}_2,R_2)$ be two ML systems.
\begin{enumerate}
\item
$({\bf M}_2,R_2)^{({\bf M}_1,R_1)}$ is a category and so an ML system.
\item
All ML transformations from $({\bf M}_1,R_1)$ to $({\bf M}_2,R_2)$ have a preorder $ \preceq$, defined by $T_1 \preceq T_2$ if there is a natural transformation $\alpha:T_1\rightarrow T_2$.
\end{enumerate}
\end{proposition}

\section{Adjunctions between ML Systems}\label{sect:mladjunction}
Recall that an adjunction between two categories ${\bf C}$ and ${\bf D}$ is given by a pair of functors $F:{\bf C}\rightarrow {\bf B}$ and $G:{\bf B}\rightarrow {\bf C}$ and
forms a functor loop, corresponding to a {\em weak form} of equivalence between ${\bf C}$ and ${\bf D}$, such that for $C\in {\bf C}_0$ and $B\in {\bf B}_0$
$$\infer={FC\rightarrow B}{C\rightarrow GB}$$ 
which is natural in $C$ and $B$. Adjunction can be defined by universal arrows (See \ref{subsect:adjoint} for the details). 

Throughout this section,  $\langle F,G,\varphi\rangle:({\bf M}_1,R_1)\rightarrow ({\bf M}_2,R_2)$ is an adjunction between two ML systems $({\bf M}_1,R_1)$
and $({\bf M}_2,R_2)$.
\subsection{Solve Problems Optimally by Adjunctions}\label{subset:optadj}
An ML system transformation $T:({\bf M}_1,R_1)\rightarrow ({\bf M}_2,R_2)$ transforms problems in ML system $({\bf M}_1,R_1)$ to $({\bf M}_2,R_2)$
as the problems transformed might be easier to solve in $({\bf M}_2,R_2)$. 
After the transformed problems being solved in $({\bf M}_2,R_2)$, 
one needs to transform the solutions back to $({\bf M}_1,R_1)$, with the structures used for the solutions being preserved, so that the original problems are solved in $({\bf M}_1,R_1)$.
Hence an ML system transformation $S: ({\bf M}_2,R_2)\rightarrow ({\bf M}_1,R_1)$ is needed.
If $T$ and $S$ are mutually inverse to each other (isomorphism) or inverse to each other up to natural isomorphism of functors (equivalence),
then $T$ and $S$ are just ``relabelling" bijectively or adding more copies of objects up to isomorphism and so it is hard to reduce the complexity of the problems by using the ML transformations $T$ and $S$
as an isomorphism could not reduce the complexity of the problem.
Categorical concept adjunction, a functor loop, provides a pipeline of transforming problems 
between ML systems $({\bf M}_1,R_1)$ and $({\bf M}_2,R_2)$ in optimal ways.

Since adjunction $\langle F,G,\varphi\rangle:({\bf M}_1,R_1)\rightarrow ({\bf M}_2,R_2)$ provides a loop and $F$ and $G$ determine each other uniquely and naturally,
$F$ produces the most efficient solutions to the problem posed by $G$.  Hence we use adjunctions to
transform ML systems and obtain optimal ways to solve ML problems.
\begin{example}
\begin{enumerate}
\item
Yoneda embedding forms part of adjunction generally.
Recall that a {\em total} category is a small category whose Yoneda embedding has a left adjoint.
Totality of a category was studied very extensively \cite{t,s,k,d}.
Many classes of categories are total,
including any category which is monadic over ${\bf Set}$, Grothendieck toposes,
locally presentable categories and so are ${\bf Grph}$ and ${\bf Top}$.
Hence Yoneda embedding of an ML system $({\bf M}, R)$ has a left adjoint. 
To calculate the left adjoint $F$ of $Y$, we consider
$$\infer=[,]{FX\rightarrow M}{X\rightarrow YM}$$ 
where $X:{\bf M}^{\textrm{op}}\rightarrow {\bf Set}$ is a set valued functor.
Since each presheaf is a colimit of representable set valued functors, one assumes $X$ is reprsentable and $FX$
is the representing object of $X$.
Hence $F$ is defined by the colimit of these reprersenting objects.
\item
A {\em forgetful functor} is a functor, defined by {\em forgetting} some structure, such that, 
forgetting composition and identities of a category to get a directed graph,
forgetting algebraic structures, e.g., monoid, group, module, to obtain a set, etc.
The left adjoint of such a forgetful functor is called {\em free} functor, such as, free category functor, free monoid functor, free group functor,
and free module functor.
\item
A monad and its $T$-algebras leads to an adjunction (see \ref{subsect:monad}) and so an optimal way to solve problems with algebraic structures. 
\item
Recall that Word2vec transformation $w2v:{\bf W}_1\rightarrow {\bf R}^n$ discused in Example \ref{item:mlsystrans_nlp}.
If $w2v$ has a left (or right) adjoint $L$ and so it is part of an adjunction to solve the word representation problem, then
for each $w\in {\bf W}$ and each $v\in {\bf R}^n$, 
$$\infer=[,]{Lv\rightarrow w}{v\rightarrow w2v(w)}$$ 
which is natural in $w$ and $v$.
Similarly, the ML system transformation loop
$$\xymatrix@R=1pc@C=1pc{
{\scriptstyle ({\bf Data},R_d)} \ar@/_1pc/[rr]_{M} && {\scriptstyle ({\bf Model},R_m)}\ar@/_1pc/[ll]_P
}$$
discussed in Example \ref{data22model} forms an adjunction if and only if for $d\in {\bf Data}$ and $m\in {\bf Model}$
$$\infer=[,]{Pm\rightarrow d}{m\rightarrow Md}$$ 
which is natural in $d$ and $m$ or
$$\infer=[,]{Md\rightarrow m}{d\rightarrow Pm}$$ 
which is natural in $d$ and $m$.
\item
Let $e:N_1\rightarrow N_2$ be a relation in an ML system $({\bf M},R)$ which has pullbacks. 
Then we have the following adjoint pair:
$$\xymatrix@C=1.6pc{ 
{\scriptstyle ({\bf M},R)/N_2} \ar@<-0.5ex>[r]_{e^*} & \ar@<-0.5ex>[l]_{e_!}{\scriptstyle ({\bf M},R)/N_1}
}$$
where $({\bf M},R)/N_2$ is the slice ML system over $N_2$ with all relations to $N_2$ being objects, 
$e_!(D, s) = es$, $e^*(C, r) = \pi_1$ which is given by the following pullback:
$$\xymatrix@C=1.6pc{ 
{\scriptstyle N_1\times_{N_2}C} \ar[r]^>>>>>>{\pi_2} \ar[d]_{\pi_1} & {\scriptstyle C}\ar[d]^r\\
{\scriptstyle N_1}\ar[r]^e & {\scriptstyle N_2}
}$$
 The unit and counit of $e_!\dashv e^*$ is given by $\eta (s:D\rightarrow N_1) =\langle s,1_C \rangle: C \rightarrow N_1 \times_{N_2} C$
 and $\varepsilon (r:C\rightarrow N_2) = \pi_2$, 
where $\pi_2$ is defined by the last pullback and $\langle s,1_C \rangle$ by the following pullback:
$$
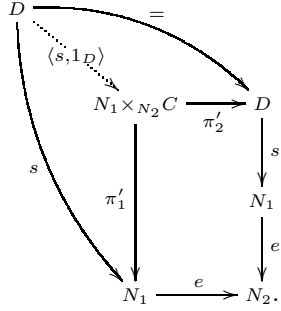
\xymatrix@C=1.6pc{ 
{\scriptstyle D} \ar@/_1pc/[dddr]_s \ar@/^1pc/[drr]^{=}\ar@{..>}[dr]|{\langle s,1_D\rangle}\\
& {\scriptstyle N_1\times _{N_2} C} \ar[dd]_{\pi'_1} \ar[r]_>>>>>{\pi'_2} & {\scriptstyle D} \ar[d]^s\\
&& {\scriptstyle N_1}\ar[d]^e\\
& {\scriptstyle N_1}\ar[r]^e & {\scriptstyle N_2}.\\
}$$
\end{enumerate}
\end{example}

\subsection{Extend Machine Learning Systems by Adding Algebra Structures}
ML systems and system transformations were formatted and reasoned categorically. 
ML objects and pipelines can have other formations other than the categorical way.
However, ML aims to not only understand and summarize the existing knowledge from data but also grow and create insights.

Recall that a monad on a category is an endo functor with monoid-like structure. 
Monads and their $T$-algebras can provide algebraic structures to the category.
An adjunction gives rise to a monad and  every monad arises this way (see \ref{subsect:monad} for the details).

\begin{definition}
A {\em monad} $T=\langle T,\eta,\mu\rangle$
on an ML system $({\bf M},R)$ is an ML system transformation $T:({\bf M},R)\rightarrow ({\bf M},R)$ and two
natural transformations
$$\eta:I\rightarrow T,\mu:T^2\rightarrow T$$
such that
$$\xymatrix@C=1.6pc{
{\scriptstyle T^3} \ar[r]^{T\mu}\ar[d]_{\mu T} & {\scriptstyle T^2}\ar[d]^{\mu}\\
{\scriptstyle T^2} \ar[r]^{\mu} & {\scriptstyle T}\\
}$$
and 
$$\xymatrix@C=1.6pc{
{\scriptstyle IT} \ar[r]^{\eta T}\ar[dr]_1 & {\scriptstyle T^2} \ar[d]|{\mu} & {\scriptstyle TI\ar[l]_{T\eta}} \ar[dl]^1\\
& {\scriptstyle T}\\
}$$
are commutative, where $I:({\bf M},R)\rightarrow ({\bf M},R)$ is the identity functor, sending $e:X\rightarrow Y$ to $e:X\rightarrow Y$.
\end{definition}
Each adjunction $\langle F,G,\varphi\rangle:({\bf M}_1,R_1)\rightarrow ({\bf M}_2,R_2)$ between ML systems gives rise to a monad
$\langle GF,\eta,G\varepsilon F\rangle$ on $({\bf M}_1,R_1)$  and very monad on $({\bf M}_1,R_1)$ arises this way.

A monad $T:({\bf M},R)\rightarrow ({\bf M},R)$ generates algebraic structures to $({\bf M},R)$.
\subsection{Link Machine Learning Objects by Universal Properties}
Universal property characterizes objects by their relations or links to other objects uniquely up to isomorphism.
Adjunctions, free objects, limits/colimits, and representable functors are the examples that are determined by
their universal property.
Universal property confirms the existence and uniqueness of a relation or map to fill in.  By the universal property
the relation filled in is a best or most efficient solution. Hence ML objects can be linked by universal property naturally.

\begin{example}
\begin{enumerate}
\item
Let $A,B,X$ be tables/neural networks related and $\langle A, B\rangle$ the join/merge of $A$ and $B$.
If there are the relations  $f:X\rightarrow A$ and $g:X\rightarrow B$ then there is a unique relation $\langle f,g\rangle:X\rightarrow \langle A, B\rangle$ such that
$$\xymatrix@C=1.6pc{ 
& {\scriptstyle X}\ar@{.>}[d]|{\langle f,g\rangle} \ar[dl]_f\ar[dr]^g\\
{\scriptstyle A}  & {\scriptstyle \langle A, B\rangle} \ar[l]_{\pi_1} \ar[r]^{\pi_2} & {\scriptstyle B}\\
}$$
commutes, where $\pi_1,\pi_2$ are projections.
\item
Given a representable functor $F:{\bf C}^{\textrm{op}}\rightarrow {\bf Set}$, if $F$ is represented by $X$ and $Y$, then
there are relations $f:X\rightarrow Y$ and $g:Y\rightarrow X$ such that $gf=1_X$ and $fg=1_Y$.
\item
The adjunctions given by Yoneda embedding (see \ref{subsect:yoneda}) and change of base (see \ref{subsect:descent}) are characterized by their universal properties, which can be used to fill out 
the gap relations between ML elements, e.g., words and chains of words in a corpus.

Let ${\bf W}_0$ be the collection of all words from a set of large corpora
and ${\bf W}_1$ the directed graph by linking words in ${\bf W}_0$ using the ordered pairs of words appearing in the large corpora.
Applying $path$ to ${\bf W}_1$, one has a directed graph with composition (category) $path({\bf W}_1)$.
So categorical notions and results are applicable to $path({\bf W}_1)$.
For example, special maps, e.g., epics and monics, are discussed as follows.

A word chain $e:w_1\rightarrow w_2$ in $path({\bf W}_1)$ is {\em epic} if
$c_1e=c_2e$ implies $c_1$ and $c_2$ have the same meaning. Dually,
a word chain $e:w_1\rightarrow w_2$ in $path({\bf W}_1)$ is {\em monic} if
$ec_1=ec_2$ implies $c_1$ and $c_2$ have the same meaning. 

$$e=\textrm{``The meaning of this sentence is defined by the succeeding chain of words"}$$
and 
$$m=\textrm{``The meaning of this sentence is determined by the preceding chain of words"}$$
are clearly epic and monic, respectively if these words are in ${\bf W}_0$.
$s=\textrm{``I love fruits"}$ is neither epic nor monic since
$$t_1s=t_2s, st_1=st_2, t_1\neq t_2$$
up to meaning, where $t_1=\textrm{``I love apples"}$ and $t_1=\textrm{``I love oranges"}$.
Reasoning $path({\bf W}_1)$ categorically
holds independent interest, is beyond the scope of the current paper, and
will be addressed separately.

Since $path({\bf W}_1)$, we have the Yoneda embedding
$${\scriptstyle Y}: {\scriptstyle path({\bf W}_1)}\rightarrow {\scriptstyle {\bf Set}^{(path({\bf W}_1))^{\textrm{op}}}}.$$
The left adjoint $L$ of $Y$  is calculated by representing a presheaf as a colimit of presentable functors.

When $path({\bf W}_1)$ has pullbacks, for each chain of word $p:w_1\rightarrow w_2$, we have an adjunction:
$$\xymatrix{ 
{\scriptstyle path({\bf W}_1)/w_2} \ar@<-0.5ex>[r]_{p^*} & \ar@<-0.5ex>[l]_{p_!} {\scriptstyle path({\bf W}_1)/w_1}
}$$
where $p_!(D, s) = ps$, $p^*(C, r) = \pi_1$ which is given by the following pullback:
$$\xymatrix{ 
{\scriptstyle w_1\times_{w_2}C} \ar[r]^{\pi_2} \ar[d]_{\pi_1} & {\scriptstyle C}\ar[d]^r\\
{\scriptstyle w_1}\ar[r]^p & {\scriptstyle w_2}
}$$
\end{enumerate}
\end{example}

\section{Conclusions}\label{sect:conclusions}
Data is from various platforms with multiple formats, noisy, and changes constantly.
ML elements and systems, driven by data and outputting new data, must be robust to the changes.
The relations between ML elements make more sense than the isolated ML elements.
We studied the ML elements that we are interested in together
as an ML system. The relations between ML elements we concerned are algebraic operations, binary relations, 
and binary relations with composition that can be reasoned categorically.
An ML system transformation between two systems is a
map between the systems, which preserves the relations we concerned. 
The ML system transformations given by quotient or clustering, representable functor and Yoneda embedding 
were highlighted and discussed by ML examples. 
ML elements were embedded to set valued functors which provide multiple relation perspectives for each ML element.
ML system transformations were linked and compared by their maps at 2-cell, natural transformations. 
ML transformations lead to fresh perspectives and uncover new insights.
Special ML system transformation loops, adjunctions between ML systems, offered the optimal way of solving ML problems.
New ML insights and structures can be obtained from universal properties and algebraic structures given by monads, which are generated from
adjunctions. 


\newpage

\appendix

The minimum requirements of the relation, directed graph, and category theory
for the paper include: binary relation, equivalence relation, equivalence class, quotient, category,
homomorphism, isomorphism, coproduct, pullback, pushout,
monic, epic, injection, initial object, functor, natural transformation, Yoneda lemma and embedding, adjunction, monad, $T$-algebra. 
For the related notions, notations, results, and a systematic introduction, the reader may consult, for instance, \cite{Ma,ahs, fs}.
\section{Binary Relations and Directed Graphs}
\subsection{Binary Relations}\label{subsect:relation}
Recall that a {\em binary relation} on a nonempty set $S$ is a subset $\rho \subseteq S\times S$, where 
$S \times S = \{(s_1, s_2) | s_1, s_2\in S\}$ is the {\em Cartesian product} of $S$ and $S$.
A binary relation $\rho$ on $S$ is {\em reflexive} if $(s, s) \in\rho$ for all $s \in S$, 
{\em symmetric} if $(s_1,s_2)\in S$ implies $(s_2,s_1)\in \rho$, 
{\em transitive} if $(s_1,s_2)\in \rho$ and $(s_2,s_3)\in \rho$ imply $(s_1,s_3)\in\rho$,
and {\em antisymmetric} if $(s_1,s_2)\in \rho$ and $(s_2,s_1)\in\rho$ imply $s_1=s_2$. 

A {\em poset} $(P,\leq)$ consists of a nonempty set $P$ and a reflexive, antisymmetric, and transitive binary relation $\leq$ on $P$.

An {\em equivalence relation} on $S$ 
is a reflexive, symmetric, and transitive binary relation on $S$. 
Let $\rho$ be an equivalence relation on $S$ and $s\in S$. The {\em equivalence class} of $s$ is $\{x\in S\;|\;(x,s)\in \rho\}$, denoted by $[s]_{\rho}$.

{\em Clustering} aims to group a set of the objects in such a way that objects in the same cluster are more similar to each other.
In a nonempty set $S$, clustering  the objects (elements) of $S$ amounts to grouping or partitioning them, which turns
out to be an equivalence relation on $S$.
\subsection{Directed Graphs}\label{subsect:graph}
Recall that a {\em (multi)directed graph} $(N, E)$ consists of a collection $N$ of {\em nodes} (or {\em vertices}), a collection $E$ of {\em edges}, 
and two functions
$$\xymatrix{
E \ar@<0.6ex>[r]^{\textrm{\scriptsize from}} \ar@<-0.6ex>[r]_{\textrm{\scriptsize to}}& N
}$$
that specify ``from" node and ``to" node of each edge $f\in E$.
Write $f\in E$ by $f:X\rightarrow Y$, where $X=\textrm{from}(f), Y=\textrm{to}(f)$.

\section{Categories, Functors, and Natural Transformations}
\subsection{Categories}\label{subsect:category}
A {\em category} ${\bf C}$ is a directed graph $(N,E)$ with identities and associative composition, which are two functions: 
$$\textrm{id}:N\rightarrow E\;\;\textrm{and}\;\;\textrm{comp}:E\times_NE\rightarrow E,$$
where $E\times_NE=\{(f,g)\;|\;\textrm{to}(f)=\textrm{from}(g)\}\subseteq E\times E$ collects all {\em composable} pairs of edges, such that
 \begin{itemize}
\item
for all edge $f:X\rightarrow Y$, $1_Yf=f1_X=f$, where $1_X=\textrm{id}(X)$,
\item
for all $f:A\rightarrow B$, $g:B\rightarrow C$, $h:C\rightarrow D$, $h(gf)=(hg)f$, where $gf=\textrm{comp}(f,g)$ for each composable .
\end{itemize}

As a category ${\bf C}$ is a directed graph, write the set of nodes and the set of edges of ${\bf C}$ by ${\bf C}_0$ and ${\bf C}_1$, respectively. 
Nodes and edges in a category are also called {\em objects} and {\em maps} of the category, respectively.

Given a category ${\bf C}$, if we flip the directions of all maps in ${\bf C}$
then we obtain its {\em dual category}, 
denoted by ${\bf C}^{\rm op}$.
Clearly, $({\bf C}^{\rm op})^{\rm op}={\bf C}$.

Given objects $X,Y\in {\bf C}_0$, write
$$\textrm{hom}_{\bf C}(X,Y)=\{f\in {\bf C}_1\;|\;\textrm{from}(f) =X\textrm{ and }\textrm{to}(f)=Y\}.$$

A {\em subcategory} ${\bf C}'$ of a category ${\bf C}$ is given 
by any subcollections of the objects and maps of ${\bf C}$, which is a category
under the from, to, composition, and identity operations of ${\bf C}$.

Given a directed graph $G=(N,E)$ and two nodes $a,b\in N$, a {\em path} from $a$ to $b$ is a sequence of composable edges $[e_1,e_2,\ldots, e_k]$ such that
$$\textrm{from}(e_1)=a,\ldots,\textrm{to}(e_i)=\textrm{from}(e_{i+1}),\ldots,\textrm{to}(e_k)=b,i=1,\ldots,k-1.$$
Each node has an empty path $[\,]$.
Each directed graph $G=(N,E)$ generates a category 
$$Path(G) = Path(N,E)=(N,\{\textrm{all paths in }G\}),$$
by considering all paths as edges, empty path as identity, path concatenation as composition. Call $Path(G)$ the {\em free category} on directed graph $G$.

On the other hand, given a category ${\bf C}$, one has a directed graph $F({\bf C})$ by forgetting identity edges and edge composition.

Some examples of categories are listed below.
\begin{enumerate}
\item
Each poset $(P,\leq)$ is a category with the elements of $P$ as its objects and  $\leq$ as maps.
\item
All sets and functions between sets form a category ${\bf Set}$.
\item
All directed graphs and graph homomorphisms between graphs form a category ${\bf Grph}$.
\item
All categories and functors between categories form a category ${\bf Cat}$.
\end{enumerate}

\subsection{Limits and Colimits}\label{section:limitsand colimits}
Limits and colimits are an example of universals.
Given a category ${\bf C}$, an ${\bf I}$-indexed {\em diagram}\index{diagram} in ${\bf C}$ is a functor
$D:{\bf I}\rightarrow {\bf C}$, where the category ${\bf I}$ is thought of as index category.
A $D$-{\em cone}\index{cone} is a natural transformation
$\phi:L\rightarrow D$, where $L:{\bf I}\rightarrow {\bf C}$ is a constant functor that sends each ${\bf I}$-map
$f:I\rightarrow J$ to a constant $1_L:L\rightarrow L$ in ${\bf C}_1$. 
Each $D$-cone can be specified by a ${\bf C}$-object $L$ together with a family of ${\bf C}_1$ elements
$(\phi_I:L\rightarrow DI)_{I\in {\rm ob}({\bf I})}$ such that $Df\phi_I=\phi_J$: 
$$\xymatrix{
& {\scriptstyle L}\ar[dl]_{\phi_I} \ar[dr]^{\phi_J}\\
{\scriptstyle DI}\ar[rr]^{Df} && {\scriptstyle DJ}
}$$
for each ${\bf I}$-map $f:I\rightarrow J$.
A {\em limit}\index{limit} of the diagram $D:{\bf I}\rightarrow {\bf C}$ is a $D$-cone $(L,\phi)$ such that 
for each other $D$-cone $(J,\psi)$ there is a unique ${\bf C}_1$ element $u:J\rightarrow L$ making the following
diagram
$$\xymatrix{
{\scriptstyle J} \ar[dr]^{\psi_I} \ar@{..>}[dd]_u\\
& {\scriptstyle DI}\\
{\scriptstyle L} \ar[ur]_{\phi_I}
}$$
commute for each ${\bf I}$-object $I$.
If ${\bf I}$ is specified by the following graphs
$$
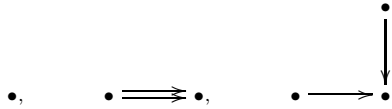
\xymatrix{
&&&& {\scriptstyle {\bullet}} \ar[d]\\
{\scriptstyle {\bullet,}} & {\scriptstyle {\bullet}}\ar@<0.3ex>[r] \ar@<-0.3ex>[r] & {\scriptstyle {\bullet,}} & {\scriptstyle {\bullet}} \ar[r] & {\scriptstyle {\bullet}}\\
}$$
then the limit of $D:{\bf I}\rightarrow {\bf C}$ is called a {\em terminal object}, an {\em equalizer}, 
a {\em pullback (square)} in ${\bf C}$, respectively. 

Explicitly, 
a ${\bf C}$-object $1$ is a {\em terminal object}\index{terminal object} provided for each ${\bf C}$-object $X$
there is a unique ${\bf C}$-map $!_X:X\rightarrow 1$.

A commutative square
$$
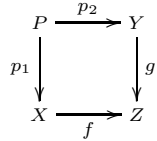
\xymatrix{
{\scriptstyle P}\ar[d]_{p_1} \ar[r]^{p_2} & {\scriptstyle Y} \ar[d]^g\\
{\scriptstyle X} \ar[r]_f & {\scriptstyle Z}
}$$
in ${\bf C}$ is called a {\em pullback (square)}\index{pullback (square)} provided given any
${\bf C}$-maps $w_1:W\rightarrow X$ and $w_2:W\rightarrow Y$ with $fw_1=gw_2$ there is
a unique ${\bf C}$-map $w:W\rightarrow P$ such that
$$p_1w=w_1\text{ and }p_2w=w_2:$$
$$\xymatrix{
{\scriptstyle W} \ar@/_/[ddr]_{w_1}\ar@/^/[drr]^{w_2} \ar@{..>}[dr]|{w}\\
& {\scriptstyle P}\ar[d]_{p_1} \ar[r]^{p_2} & {\scriptstyle Y} \ar[d]^g\\
& {\scriptstyle X} \ar[r]_f & {\scriptstyle Z}
}$$

For two parallel ${\bf C}$-maps $f,g:X\rightarrow Y$, the {\em equalizer}\index{equalizer}
of $f$ and $g$ is a ${\bf C}$-map $e:E\rightarrow X$ such that $fe=ge$
and $e$ is unique with this property:
if a ${\bf C}$-map $z:Z\rightarrow X$ is such that $fz=gz$
then there is a unique ${\bf C}$-map $d:Z\rightarrow E$ such that
$ed=z$:
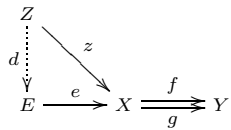
$$\xymatrix{
{\scriptstyle Z} \ar@{..>}[d]_{d} \ar[dr]^{z}\\
{\scriptstyle E} \ar[r]^e & {\scriptstyle X}\ar@<0.3ex>[r]^f \ar@<-0.3ex>[r]_g & {\scriptstyle Y}\\
}$$

If maps between $D$-cones are defined properly, then limits can be characterized as terminal objects in 
the category of all $D$-cones.

The dual notions of cone, limit, terminal object, pullback (square), equalizer are {\em cocone},\index{cocone} {\em colimit},\index{colimit}
{\em initial object},\index{initial object} {\em pushout (square)},\index{pushout (square)} and {\em coequalizer},\index{coequalizer} respectively.

\subsection{Functors and Natural Transformations}\label{subsect:functor}
Let ${\bf C}$ and ${\bf D}$ be categories. A {\em functor}  $F:{\bf C}\rightarrow {\bf D}$  is a structure preserving function $F$ between ${\bf C}$ and ${\bf D}$, which maps ${\bf C}_i$ to ${\bf D}_i$: 
$F{\bf C}_i\subseteq {\bf D}_i$, $i =0, 1$,  such that
\begin{enumerate}
\item
for each $X\in {\bf C}_0$, $F1_X = 1_{FX}$;
\item
for each composable edge pair $(f,g)$ in ${\bf C}_1$, $(Ff,Fg)$ is a composable pair and $F(gf)=FgFf$.
\end{enumerate}

 A functor $F:{\bf C}\rightarrow {\bf D}$ is {\em full} ({\em faithful}) if each function
 $$F:\textrm{hom}_{\bf C}(X,Y)\rightarrow \textrm{hom}_{\bf D}(FX,FY),$$
 sending $f:X\rightarrow Y$ to $Ff:FX\rightarrow FY$, is surjective (injective) for all $X,Y\in C_0$.

Let $F,G:{\bf C}\rightarrow {\bf D}$ be two functors. 
A {\em natural transformation}
$\alpha$ from $F$ to $G$, written as $\alpha:F\rightarrow G$, is specified by
an operation which assigns each object $X$ of ${\bf C}$ 
a map $\alpha_X:FX\rightarrow GX$ such that for each $f:X\rightarrow Y$ in ${\bf C}_1$
$$\xymatrix{ 
{\scriptstyle FX} \ar[d]_{Ff} \ar[r]^{\alpha_X} & {\scriptstyle GX} \ar[d]^{Gf}\\
{\scriptstyle FY} \ar[r]^{\alpha_Y} & {\scriptstyle GY}
}$$
commutes in ${\bf D}$.
Natural transformations are maps between functors.
A natural transformation $\alpha$ is called {\em a natural isomorphism},
denoted by $\alpha:F\cong G$, if each {\em component} 
$\alpha_X$ is an isomorphism.

An {\em equivalence} between categories ${\bf C}$ and ${\bf D}$ 
is a pair of functors $S:{\bf C}\rightarrow {\bf D}$ and  
$T:{\bf D}\rightarrow {\bf C}$ together with natural isomorphisms
$1_{\bf C}\cong TS$ and $1_{\bf D}\cong ST$.

\subsection{Quotient Categories}\label{subsect:quotientcat}
Given nonempty category ${\bf C}$,  one may cluster the maps (edges) or objects (nodes) of ${\bf C}$ 
to obtain the new {\em quotient category} ${\bf C}/\rho$ 
with respect to an equivalence relation $\rho$ on maps or objects such that composition under $\rho$ is well defined.

Let $\rho$ be an equivalence relation on ${\bf C}_1$. $\rho$ is a {\em congruence}
if  for $f,g\in \textrm{hom}_{\bf C}(X,Y)$ such that $kfh$ and $kgh$ are composable and $(f,g)\in\rho$ imply 
$(kfh,kgh)\in\rho$.  {\em Quotient category} ${\bf C}/\rho$ is defined by
\begin{itemize}
\item
$({\bf C}/\rho)_0={\bf C}_0$,
\item
$({\bf C}/\rho)_1=\{[f]_{\rho}\;|\;f\in {\bf C}_1$, where $[f]_{\rho} : X\rightarrow Y$
as a representative member of the equivalence class of $f:X\rightarrow Y$ in ${\bf C}$,
\item
{\bf identities}: $[1_X]_{\rho}:X \rightarrow X$,
\item
{\bf composition}: $[g]_{\rho}[f]_{\rho}=[gf]_{\rho}$.
\end{itemize}

There is an obvious canonical functor
$$Q_{\rho}:{\bf C}\rightarrow {\bf C}/\rho$$
sending $f:X\rightarrow Y$ to $[f]_{\rho}:X\rightarrow Y$.
Each functor $F:{\bf C}\rightarrow {\bf D}$, which preserves the congruence equivalence relation $\rho$: $(f,g)\in \rho$ implies $Ff=Fg$,
factors through $Q_{\rho}$, followed by a unique functor $F_{\rho}:{\bf C}/\rho\rightarrow {\bf D}$
$$\xymatrix{
{\scriptstyle {\bf C}}\ar[rr]^F\ar[dr]_{Q_{\rho}} && {\scriptstyle {\bf D}}\\
& {\scriptstyle {\bf C}/\rho} \ar[ur]_{F_{\rho}}
}$$

Clustering the nodes (objects) of ${\bf C}$ by an equivalence relation, can be more complicated
since new composition on the equivalence classes of edges (maps) needs to be well defined.
However, one may first quotient the directed graph by identifying nodes then 
add all paths to the quotient directed graph to form a category.

\subsection{Presheaves, Representable Functors, and Yoneda Embedding}\label{subsect:yoneda}
A category is {\em small} if its objects and maps are sets and so it is {\bf Set}-enriched as the maps between each pair $(X,Y)$ of objects form a hom set $\textrm{hom}(X,Y)$. 
Let ${\bf  C}$ be s small category.

A {\em presheaf} on a small category ${\bf C}$ is a set valued functor $F:{\bf C}^{\textrm{op}}\rightarrow {\bf Set}$.
\begin{theorem}\label{thm:presheaves}
Given a small category ${\bf C}$, each category of presheaves on ${\bf C}$ is complete and cocomplete and both limits and colimits are computed point wise.
\end{theorem}
Since ${\bf Grph}\cong {\bf Set}^{\bf 2}$, one has:
\begin{corollary}
${\bf Grph}$ has both limits and colimits, being computed point wise.
\end{corollary}
A set valued functor $S:{\bf C}^{\textrm{op}}\rightarrow {\bf Set}$ is {\em representable}  or {\em represented} by $A\in {\bf C}_0$ 
if it is naturally isomorphic to a hom functor $\textrm{hom}_{\bf C}(-,A)$ for $A\in {\bf C}_0$.

\begin{theorem}
Each presheaf is a colimit of representable set valued functors.
\end{theorem}
The Yoneda lemma states that natural transformations from a representable functor $\textrm{hom}_{\bf C}(-,X)$ to a set valued functor $S:{\bf C}^{\textrm{op}}\rightarrow {\bf Set}$
is in natural bijection with $SX$:

\begin{proposition} [Yonneda Lemma]
Let ${\bf C}$ be a small category, $S:{\bf C}^{\textrm{op}}\rightarrow {\bf Set}$ a functor, and $X\in {\bf C}_0$. Then
there is a natural isomorphism
$$\{\alpha:\textrm{hom}_{\bf C}(-,X) \rightarrow S\}\rightarrow SX.$$
\end{proposition}

\begin{proposition} [Yonneda Embedding]
There is a full and faithful embedding 
$$Y :{\bf C}\rightarrow {\bf Set}^{{\bf C}^{\textrm{op}}},$$
taking $f:X\rightarrow Y$ to $\textrm{hom}_{\bf C}(-,f):\textrm{hom}_{\bf C}(-,X)\rightarrow \textrm{hom}_{\bf C}(-,Y)$, where $\textrm{hom}_{\bf C}(-,f)(e)=fe$.
\end{proposition}

Representing objects of a representable functor are unique up to isomorphic.
\begin{corollary}
$YA\cong YB$ if and only if $A\cong B$.
\end{corollary}

Morita theory shows that one can study a ring $R$ by investigating the category
of all $R$-modules, all module structures associated to $R$. Similarly, we study an ML element $E$ using the relations around $E$ by Yoneda Lemma.

\section{Adjoints and Monads}\label{sect:adjmonads}
\subsection{Adjoints}\label{subsect:adjoint}
Recall that an {\em adjunction} from {\bf C} to {\bf D} is a triple 
$\langle F,G,\varphi\rangle:{\bf C}\rightarrow {\bf D}$, where $F$ and $G$ are functors:
$$\xymatrix{
{\scriptstyle {\bf C}} \ar@<0.5ex>[r]^F & {\scriptstyle {\bf D}} \ar@<0.5ex>[l]^G
}$$
and $\varphi$ is a function which assigns to each pair of objects 
$C\in {\bf C}, D\in {\bf D}$ a bijection of sets
$$\varphi=\varphi _{C,D}:{\bf D}(FC,D)\cong {\bf C}(C,GD)$$

\noindent which is natural in $C$ and $D$:
$$\infer=[.]{C\rightarrow GD}{FC\rightarrow D}$$

If $F:{\bf C}\rightarrow {\bf D}$ is a functor and $D\in {\bf D}_0$,
a {\em universal arrow} from $D$ to $F$ is a pair
$(C,u)$ with $C\in {\bf C}_0$ and
$u:D\rightarrow FC$ being in ${\bf D}_1$ such that
for each pair $(C',f)$ with $C'\in {\bf C}_0$
and $f:D\rightarrow FC'\in {\bf D}_1$ there is a unique 
$f^*:C\rightarrow C'$ in ${\bf C}_1$ such that
$$\xymatrix{
{\scriptstyle D}\ar[dr]_{f} \ar[r]^u & {\scriptstyle FC}\ar[d]^{Ff^*} & {\scriptstyle C} \ar@{..>}[d]|{f^*}\\
& {\scriptstyle FC'} & {\scriptstyle C'} 
}$$
commutes.
Equivalently, $u:D\rightarrow FC$ is universal from $D$ to $F$
provided that the pair $(C,u)$ is an initial object 
in the comma category $(D\downarrow F)$ that has maps $D\rightarrow FC$ as its objects.

If $G:{\bf D}\rightarrow {\bf C}$ is a functor and $C\in {\bf C}_0$,
dually, a {\em universal arrow} from $G$ to $C$ is a pair
$(D,v)$ with $D\in {\bf D}_0$ and $v:GD\rightarrow C$ in ${\bf C}_1$
such that for each pair $(D',f)$ with $D'\in {\bf D}_0$
and $f:GD'\rightarrow C\in {\bf C}_1$ there is a unique 
$f^{\sharp}:D'\rightarrow D$ in ${\bf D}_1$ making
$$\xymatrix{
{\scriptstyle D'}\ar@{..>}[d]|{f^{\sharp}} & {\scriptstyle GD'} \ar[d]_{Gf^{\sharp}} \ar[dr]^f\\
{\scriptstyle D} & {\scriptstyle GD}\ar[r]^v & {\scriptstyle C}
}$$
commute.

By [\cite{Ma}, p.83, Theorem 2], each adjunction $\langle F,G,\varphi\rangle:{\bf C}\rightarrow 
{\bf D}$ is completely determined by
one of five conditions.
Here we only record some of them, which we shall use in this thesis:

\begin{description}
\item{(ii)} 
{\em The functor $G:{\bf D}\rightarrow {\bf C}$ and for each $C\in {\rm ob}({\bf C})$
a $F_0(C)\in {\rm ob}({\bf C})$ and a universal arrow 
$\eta_C:C\rightarrow GF_0(C)$ from $C$ to $G$. Then the functor $F$ has object function $F_0$
and is given by sending $f:C\rightarrow C'$ to
$GF(f)\eta_C=\eta_{C'}f$.}
\item{(iv)} 
{\em The functor $F:{\bf C}\rightarrow {\bf D}$ and for each $D\in {\rm ob}({\bf D})$
a $G_0(D)\in {\rm ob}({\bf C})$ and a universal arrow
$\varepsilon_D:FG_0(D)\rightarrow D$ from $F$ to $D$.}
\item{(v)} 
{\em Functors F, G and natural transformations 
$\eta: 1_{\bf C}\rightarrow GF$ and 
$\varepsilon: FG\rightarrow 1_{\bf D}$ such that $G\varepsilon \cdot \eta G=1_G$
and $\varepsilon F\cdot F\eta =1_F$.}
\end{description}

\noindent
Hence we often denote the adjunction $\langle F,G,\varphi\rangle:{\bf C}\rightarrow {\bf D}$ by
$(\eta,\varepsilon):F\dashv G:{\bf C}\rightarrow {\bf D}$
or by $\langle F,G,\eta,\varepsilon\rangle:{\bf C}\rightarrow {\bf D}$.
In this case, we say that $F$ is a {\em left adjoint}
to $G$ or $G$ is a {\em right adjoint} to
$F$ and that $F$ has a right adjoint $G$ and $G$ has a left adjoint $F$.
We also say that $F\dashv G$ is an {\em adjoint pair}.

Given a directed graph $G=(N,E)$, a category ${\bf C}$, one has
$$\infer=[]{\textrm{category functor }Path(G)\rightarrow {\bf C}}{\textrm{graph homomorphism }G\rightarrow U({\bf C})}$$
and so there is an adjunction:
$$\xymatrix{
{\scriptstyle {\bf Grph}} \ar@<0.6ex>[r]^{Path} & \ar@<0.6ex>[l]^{U}{\scriptstyle {\bf Cat}}.
}$$

\subsection{Monads}\label{subsect:monad}
In algebra, a {\em monoid} $M$ is a semigroup with an identity element. It may be viewed as a set with two operations: 
unit $\eta:1\rightarrow M$ and composition $\mu: M\times M\rightarrow M$ such that
$$\xymatrix{
{\scriptstyle M\times M\times M}\ar[rr]^{1\times \mu} \ar[d]_{\mu\times 1} && {\scriptstyle M\times M}\ar[d]^{\mu}\\
{\scriptstyle M\times M}\ar[rr]^{\mu} && {\scriptstyle M}\\
}$$
and 
$$\xymatrix{
{\scriptstyle 1\times M} \ar[rr]^{\eta\times 1} \ar[drr]_{\pi_2} && {\scriptstyle M\times M} \ar[d]_{\mu} && {\scriptstyle M\times 1}\ar[dll]^{\pi_1}\ar[ll]_{1\times \eta}\\
&& {\scriptstyle M}\\
}$$
are commutative, where the object 1 is the one-point set $\{0\}$,
the morphism 1 is an identity map, and where $\pi_1$ and
$\pi_2$ are projections.

\begin{definition}\label{def:monad}
A monad $T=\langle T,\eta,\mu\rangle$
on a category {\bf C} consists of an endo functor $T:{\bf C}\rightarrow {\bf C}$ and two
natural transformations
$$\eta:I\rightarrow T,\mu:T^2\rightarrow T$$
such that
$$\xymatrix{
{\scriptstyle T^3} \ar[r]^{T\mu}\ar[d]_{\mu T} & {\scriptstyle T^2}\ar[d]^{\mu}\\
{\scriptstyle T^2} \ar[r]^{\mu} & {\scriptstyle T}\\
}$$
and 
$$\xymatrix{
{\scriptstyle IT} \ar[r]^{\eta T}\ar[dr]_1 & {\scriptstyle T^2} \ar[d]|{\mu} & {\scriptstyle TI}\ar[l]_{T\eta} \ar[dl]^1\\
& {\scriptstyle T}\\
}$$
are commutative, where $I:{\bf C}\rightarrow {\bf C}$ is the identity functor.
\end{definition}

If $\langle F,G;\eta,\varepsilon\rangle:{\bf C}\rightarrow {\bf B}$ is an adjunction, then
\mbox{$\langle GF,\eta,G\varepsilon F\rangle$} is a monad on {\bf C} 
(see \cite{Ma}, p.138). In fact, every monad arises this way.

\begin{definition} 
Let $\langle T,\eta,\mu\rangle$ be a monad on ${\bf C}$,  {\em a $T$-algebra}
$\langle C,\xi\rangle$ is a pair consisting of an object $C\in {\bf C}$ and 
a map $\xi:TC\rightarrow C$ in ${\bf C}$ 
such that 
$$\xymatrix{ 
{\scriptstyle T^2C}\ar[r]^{T\xi} \ar[d]_{\mu_C} & {\scriptstyle TC}\ar[d]^{\xi}\\
{\scriptstyle TC}\ar[r]^{\xi} & {\scriptstyle C}
}$$
and
$$\xymatrix{ 
{\scriptstyle C} \ar[rd]_1 \ar[r]^{\eta _C} & {\scriptstyle TC} \ar[d]^{\xi}\\
& {\scriptstyle C}
}$$
are commutative. {\em A map} 
$f:\langle C,\xi\rangle\rightarrow \langle D,\zeta\rangle$ of $T$-algebras
is a map $f:C\rightarrow D$ of ${\bf C}$ such that
$$\xymatrix{ 
{\scriptstyle TC}\ar[d]_{\xi}\ar[r]^{Tf} & {\scriptstyle TD}\ar[d]^{\zeta}\\
{\scriptstyle C}\ar[r]^f & {\scriptstyle D}
}$$
commutes.
\end{definition}

Every monad is determined by its $T$-algebras, as specified by the following
theorem:

\begin{theorem}
If $\langle T,\eta,\mu\rangle$ is a monad in {\bf C}, then all $T$-algebras
and their maps form a category ${\bf C}^T$,
called the {\em Eilenberg-Moore category} of the monad $T$ over the category
${\bf C}$. There is an adjunction
$$\langle F^T,G^T;\eta^T,\varepsilon^T\rangle:{\bf C}\rightarrow {\bf C}^T,$$
where $G^T:{\bf C}^T\rightarrow {\bf C}$ is the obvious forgetful functor
and $F^T$ is given by
$$\xymatrix{ 
\ar @{} [drr]|{\mapsto}
{\scriptstyle C} \ar[d]_{f}  \ar @{} [rr]|{\mapsto} & & {\scriptstyle \langle T(C),\mu_C\rangle} \ar[d]^{T(f)}\\
            {\scriptstyle D} \ar @{} [rr]|{\mapsto} & & {\scriptstyle \langle T(D),\mu_{D}\rangle}
}$$
Furthermore, $\eta^T=\eta$ and $\varepsilon^T_{\langle C,\xi\rangle}=\xi$ for each 
$T$-algebra $\langle C,\xi\rangle$. The monad
defined in {\bf C} by this adjunction is $\langle T,\eta,\mu\rangle$.
\end{theorem}
\begin{proof}
See \cite{Ma}, pp.140-141.
\end{proof} 

Some examples $T$-algebras are as follows.
\begin{example} 
\begin{enumerate}
\item
{\em Complete Lattices.}
The ${\cal P}$-algebras are free complete lattice:
$(X,\xi:{\cal P}X\rightarrow X)$ where $\xi(S)$ is the supremum of $S\subseteq X$
and the the maps are maps which preserve arbitrary suprema.
\item
{\em Modules.} Let $R$ be a unital ring. Then
$$T_R(A)=A\otimes R, \eta_A:A\rightarrow A\otimes R:a\mapsto a\otimes 1$$
and 
$$\mu_A:(A\otimes R)\otimes R\rightarrow A\otimes R:(a\otimes r_1)\otimes 
r_2\mapsto a\otimes (r_1r_2)$$
for every abelian group $A$,
give a monad $(T_R,\eta,\mu)$ on the category {\bf Ab} of all 
abelian groups, and ${\bf Ab}^{T_R}$ is the category \mbox{{\bf Mod}-$R$} of 
right  \mbox{$R$-modules}.
\item
$Upath: {\bf Grph}\rightarrow {\bf Grph}$ is a monad and ${\bf Grph}^{Upath}\cong {\bf Cat}.$
\item
{\em Group Actions.} Let $G$ be a group. Then ${\bf Set}^{T_G}$ is the 
category ${\bf Set}^G$ of $G$-sets, where the monad  
\mbox{$\langle T_G,\eta,\mu\rangle$} on {\bf Set} is defined by
$$T_G(X)=G\times X,\eta_X:X\rightarrow G\times X:x\mapsto (1_G,x),$$
and
$$\mu_X:G\times (G\times X)\rightarrow G\times X:(g_1,(g_2,x))
\mapsto (g_1g_2,x).$$
\end{enumerate}
\end{example}

More generally, every variety of universal algebra is the category of 
$T$-algebras over {\bf Set}, where $TX$ is the the underlying set of the free
algebras over $X$.

\begin{theorem}[Beck's Theorem, comparison of adjunctions with algebras]\label{theorem:comp}
Let $\langle F,G;\eta,\varepsilon\rangle:{\bf C}\rightarrow {\bf B}$ be
an adjunction and \mbox{$T=\langle GF,\eta,G\varepsilon F\rangle$} 
the induced monad. Then there is a unique functor $K: {\bf B}\rightarrow {\bf C}^T$ given by
$$KB=\langle GB,G\varepsilon_B\rangle,Kf=Gf:\langle GB,G\varepsilon_B\rangle\rightarrow
\langle GB',G\varepsilon_{B'}\rangle$$
such that $G^TK=G$ and $KF=F^T:$
$$\xymatrix{ 
{\scriptstyle {\bf B}}\ar[rr]^K\ar@<.5ex>[dr]^G & & 
{\scriptstyle {\bf C}^T}\ar@<.5ex>[dl]^{G^T}\\
& {\scriptstyle {\bf C}} \ar@<.5ex>[lu]^F \ar@<.5ex>[ur]^{F^T} &
}$$
\end{theorem}
\begin{proof}
See \cite{Ma}, pp.142-143.
\end{proof} 

\begin{definition}
$G$ is {\em monadic (premonadic)} if the comparison functor $K$, defined in Beck's Theorem, is an equivalence of categories (full and faithful). 
\end{definition}

\subsection{Descent and Change of Base}\label{subsect:descent}
Descent theory plays an important role in the development of modern algebraic
geometry by Grothendieck \cite{Gr1,Gr2}. Generally
speaking, it deals with the problem of which morphisms in a given “structured”
category allow for change of base under minimal loss of information, and how to
compensate for the occurring loss, such morphisms are called 
effective descent morphisms. 

Let ${\bf C}$ be a category with pullbacks, and let $p : E \rightarrow B$ be a morphism in ${\bf C}$.
Then we have the following adjoint pair:
$$\xymatrix{ 
{\scriptstyle {\bf C}/B} \ar@<-0.5ex>[r]_{p^*} & \ar@<-0.5ex>[l]_{p_!}{\scriptstyle {\bf C}/E}
}$$
where $p_!(D, s) = ps$, $p^*(C, r) = \pi_1$ which is given by the following pullback:
$$\xymatrix{ 
{\scriptstyle E\times_BC} \ar[r]^{\pi_2} \ar[d]_{\pi_1} & {\scriptstyle C}\ar[d]^r\\
{\scriptstyle E}\ar[r]^p & {\scriptstyle B}
}$$
 The unit and counit of $p_!\dashv p^*$ is given by $\eta (s:C\rightarrow E) =\langle s,1_C \rangle: C \rightarrow E \times_B C$
 and $\varepsilon (r:C\rightarrow B) = \pi_2$, respectively.

Applying Beck's Theorem to the adjunction $p_!\dashv p^*:{\bf C}/E\rightarrow {\bf C}/B$, one has the following commutative diagram
$$\xymatrix{ 
{\scriptstyle {\bf C}/B}\ar[rr]^K\ar@<.5ex>[dr]^{p^*} & & 
{\scriptstyle ({\bf C}/E)^T}\ar@<.5ex>[dl]^{(p^*)^T}\\
& {\scriptstyle {\bf C}/E} \ar@<.5ex>[lu]^{p_!} \ar@<.5ex>[ur]^{(p_!)^T} &
}$$
where $T=p^*p_!$.

\begin{definition}
$p:E\rightarrow B$ is {\em effective descent (descent)} if the comparison functor $K$, defined in Beck's Theorem, is an equivalence of categories (full and faithful). 
\end{definition}

Dually, one has the dual of the above change of base.
\end{document}